\def\year{2018}\relax
\documentclass[letterpaper]{article} 
\usepackage{aaai18}  
\usepackage{times}  
\usepackage{helvet}  
\usepackage{courier}  
\usepackage{url}  
\usepackage{graphicx}  
\usepackage{amssymb}
\usepackage{url}
\usepackage{amsthm}
\usepackage{amsmath} 
\usepackage{amssymb}  

\DeclareMathOperator*{\argmin}{argmin}
\newcommand{\PreserveBackslash}[1]{\let\temp=\\#1\let\\=\temp}
\newcommand{\norm}[1]{\lVert#1\rVert}
\newtheorem*{Pf1}{Proof of Proposition 1}
\newtheorem*{Pf2}{Proof of Proposition 2}
\newtheorem*{Pft}{Proof of Theorem 1}

\usepackage{graphicx}
\usepackage{subfigure}
\usepackage{stfloats}
\usepackage[linesnumbered, algoruled]{algorithm2e}
\usepackage{multirow}
\usepackage{array}
\usepackage{threeparttable}
\usepackage{tabularx, booktabs}
\newtheorem{thm}{Theorem}

\newtheorem{prop}{Proposition}

\newtheorem{lem}{Lemma}
\newtheorem{rem}{Remark}
\newtheorem{asm}{Assumption}
\begin{document}
	%
	\title{From Common to Special: When Multi-Attribute Learning Meets Personalized Opinions}
	\author
	{
		Zhiyong Yang,\textsuperscript{1,2} Qianqian Xu,\textsuperscript{1}  Xiaochun Cao,\textsuperscript{1} Qingming Huang \textsuperscript{3,4} \thanks{The corresponding author.}\\
	\textsuperscript{1}{SKLOIS, Institute of Information Engineering, Chinese Academy of Sciences, Beijing, China}\\
	\textsuperscript{2}{School of Cyber Security, University of Chinese Academy of Sciences, Beijing, China}\\
	\textsuperscript{3}{University of Chinese Academy of Sciences, Beijing, China}\\
   \textsuperscript{4}{Key Lab of Intell. Info. Process., Inst. of Comput. Tech., CAS, Beijing, China}\\
		\{yangzhiyong, xuqianqian, caoxiaochun\}@iie.ac.cn, qmhuang@ucas.ac.cn	\\
	}
	\maketitle
	\begin{abstract}
		Visual attributes, which refer to human-labeled semantic annotations, have gained increasing popularity in a wide range of real world applications. Generally, the existing attribute learning methods fall into two categories: one focuses on learning user-specific labels separately for different attributes, while the other one focuses on learning crowd-sourced global labels jointly for multiple attributes. However, both categories ignore the joint effect of the two mentioned factors: the personal diversity with respect to the global consensus; and the intrinsic correlation among multiple attributes. To overcome this challenge, we propose a novel model to learn user-specific predictors across multiple attributes. In our proposed model, the diversity of personalized opinions and the intrinsic relationship among multiple attributes are unified in a common-to-special manner. To this end, we adopt a three-component decomposition.  Specifically, our model integrates a common cognition factor, an attribute-specific bias factor and a user-specific bias factor.  Meanwhile Lasso and group Lasso penalties are adopted to leverage efficient feature selection. Furthermore, theoretical analysis is conducted to show that our proposed method could reach reasonable performance. Eventually, the empirical study carried out in this paper demonstrates the effectiveness of our proposed method.
	\end{abstract}
	\section{Introduction}
	\begin{figure}
		\centering
		\subfigure{
			\includegraphics[scale=0.2]{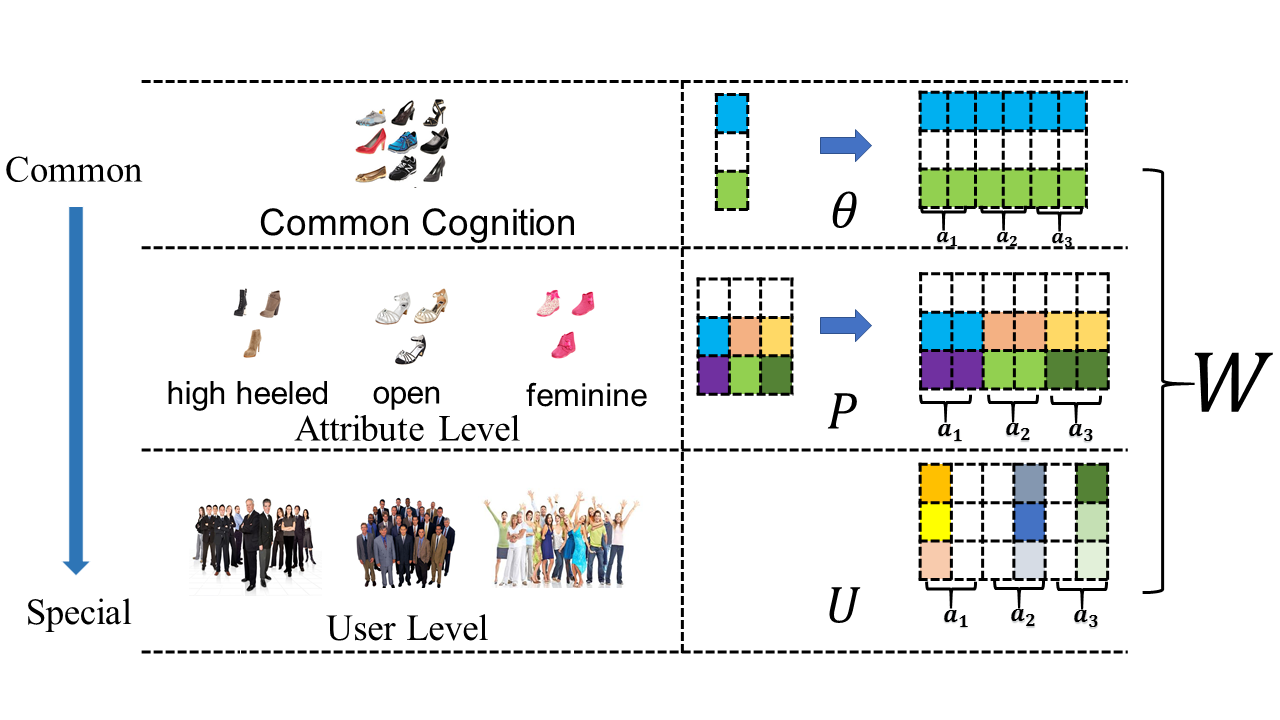}}
		\caption{Illustration of the three-component decomposition. Here $a_1$, $a_2$ and $a_3$ are the three mentioned attributes: high heeled, open and feminine. We assume that, for each attribute, there are two annotators who labeled the corresponding images. Note that we extend $\boldsymbol{\theta}$ and $\boldsymbol{P}$ to match the size of $\boldsymbol{U}$.}
	\end{figure}
	Visual attributes, which describe human labeled properties (like \textit{open}, \textit{fashionable}) for a given image, have shown its great potential as a mid-level semantic cue to enhance a variety of applications including face verification \cite{face1}, person re-identification\cite{reid1,reid2}, and zero-shot learning \cite{zeroshot1,zeroshot2,zeroshot3}, \textit{etc}. Generally speaking, there are two types of attributes:  i) binary attributes express whether a property is absent or present in a given image (like \textit{A is/is not open}); ii) relative attributes show the strength of an attribute conveyed in one image with respect to another image (like \textit{A is more/similarly/less open than B}) \cite{rel}. \\
	\indent On one hand, the attribute predictors are often trained with the crowd-sourced global labels. The justification of such an approach is that there is only one unique ground truth and the majority of annotators have to recognize this\textit{ ``correct answer''} simultaneously. However, different annotators might very well have distinct preferences, such that participants of the crowdsourced experiments might vote under different criteria or conditions. It might be misleading to merely look at a global consensus while ignoring personal diversity.
	On the other hand, practical visual applications often involve simultaneously learning multiple attributes together. In such a case, different attributes may intrinsically share some common patterns. One reason is that these attributes convey similar semantic meaning. Another reason is that they use common subsets of low level image features. In that sense, training multiple attribute predictors independently might not be an appropriate protocol.\\
	\indent Based on the discussion above, our goal is to solve two problems simultaneously in this paper: 1) learning user-specific attributes, 2) learning multiple attributes together with their shared information.\\
	\indent For 1), \cite{user1} regard user-specific attribute learning as an adaption process. In this work, a generate model is first trained based on a large pool of crowd-sourced labels. Then a small user-specific dataset is employed to adapt the generic model to user-specific predictors. Meanwhile, \cite{user2} argue that one attribute may fit to different shades (interpretations) for different groups of persons. Correspondingly, the authors proposed an automatic shade discovery method to leverage group-wise user-specific attributes. Though these existing works are designed to deal with user-specific attributes, they neglect the mutual interactions between different attributes. \\
	\indent Multi-task learning framework is well known as a standard solution for 2). Recently, many efforts have been made to improve multi-task learning.
	\cite{ASO} proposed an alternating structure optimization algorithm to decompose the predictive model of each task into two components: the task-specific feature mapping and task-shared feature mapping.
	For robust multi-task learning, \cite{robust1} proposed a corresponding method with a low-rank structure and a column-wise sparse structure.
	In \cite{robust2}, the low rank structure proposed in \cite{robust1} was replaced by a row-wise sparse structure to leverage selection of a common subset of features. Since the tasks from the same group are closer to each other than those from a different group, \cite{clustering1} proposed a clustering based multi-task learning framework. Motivated by the fact that the tasks should be related in terms of subsets of features, \cite{cocluster} proposed a novel multi-task learning method via task-feature co-clustering. As for applications, \cite{trajectory} proposed a robust dynamic multi-task method for trajectory regression. \\
	\indent There are also some existing works that focus on applying multi-task frameworks to attribute learning. One typical way to do this is to extend the existed multi-task algorithms to match attribute  learning. For instance, in \cite{relMTL1} the model proposed in \cite{robust2} was generalized to learn multiple relative attributes with their shared information. Meanwhile, some works employ deep learning methods to solve this problem by partially sharing the learned weights among different attributes. \cite{app4} proposed a multi-task restricted boltzman machine so as to learn a shared feature representation for multiple facial attribute learning. \cite{app1} incorporates identity and human attributes in learning discriminative face representations through a multi-task method. A deep multi-task learning approach was proposed in \cite{app3} to jointly estimate multiple heterogeneous attributes
	from a single face image. \cite{app2} also proposed a multi-task deep convolutional neural network with an auxiliary network at the top to capture attribute relationships. Though these works have successfully improved attribute learning with multi-task models, as was mentioned previously in this section, they all employ crowdsourced labels to train attribute predictors and ignore the disagreement among users.\\
	\indent Note that, except learning global labels for multiple attributes, multi-task frameworks are also suitable for learning user specific attributes where global patterns are necessary for capturing the public opinion, and task-specific patterns are indispensable as well for capturing user bias toward that public opinion. With such belief in mind, different with most of the previous works which partially met the requirement of our goal, we propose a hierarchical multi-task framework where task relationships are modeled on both the attribute level and the user level. \\
	\indent The main contributions of this paper are two-fold:
	\begin{itemize}
		\item To match the hierarchical nature of the underlying problem, we propose a common-to-special decomposition of the model weights, which captures the general cognition patten, attribute level bias and user specific bias, respectively. An optimization method is established based on the accelerated proximal gradient method.
		\item Theoretical analysis is performed in this paper. The corresponding results show that  our proposed algorithm could attain reasonable performance.
	\end{itemize}
	
	\section{Methodology}
	In this section, we'll present an attribute learning method to learn user specific labels across multiple attributes.
	We first introduce the notations used in this paper. Secondly, we propose our model formulation, which includes a common-to-special decomposition of the model weights and the corresponding objective function. Thirdly, we introduce our optimization method based on the accelerated proximal gradient method. Finally, the theoretical analysis is carried out to show the performance bound of our method.
	\subsection{Notations}
	\indent In this paper, scalars, vectors, and matrices are denoted as lowercase letters ($a$), bold lower case letters ($\boldsymbol{a}$), and bold upper case letters ($\boldsymbol{A}$). $\boldsymbol{X}_k$ denotes the $k$th row of $\boldsymbol{X}$.  $x_{ij}$ denotes the $(i,j)$ entry of a matrix $\boldsymbol{X}$. $\mathbb{P}(\cdot)$ denotes a probability measure. $[a]$ denotes the set :$\{1,2,\cdots,a\}$. Given an  index set $\mathcal{I}$, $\boldsymbol{A}^{\mathcal{I}}$ denotes a matrix that contains all the corresponding rows of $\boldsymbol{A}$, while $\boldsymbol{a}^\mathcal{I}$ represents the vector that contains the corresponding elements of vector $\boldsymbol{a}$.   $\norm{\cdot}_p$ denotes the $\ell_p$ norm : $\norm{x}_p ={(\sum\limits_{i}x_i^p)}^{(1/p)}$. $\norm{\cdot}_{p,q}$ denotes the $\ell_{p,q}$ norm :$\Big(\sum\limits_{i}\big(\sum\limits_{j}(x_{ij}^q)^{(1/q)}\big)^p\Big)^{(1/p)}$. $\norm{\cdot}$ denotes the $\ell_2$ norm. $\left<\cdot,\cdot\right>$ denotes the inner product for two matrices or two vectors. If $f(x) = o(g(x))$, we have $\lim\limits_{x \rightarrow 
		+\infty} \dfrac{f(x)}{g(x)}=0$ . $\oplus$ denotes the direct sum of two matrices.\\
	\subsection{Model Formulation}
	\indent Assume that we have $n_a$ attributes to be evaluated, and that, for the $i$th attribute, we are given user-specific labels from $n_{u_i}$ different workers. Then the training data could be represented as:
	\begin{equation*}
	\begin{split}
	\mathcal{T} = \left\{(\boldsymbol{X}^{(1,1)}, \boldsymbol{y}^{(1,1)}), \cdots, (\boldsymbol{X}^{(n_a,n_{u_{n_a}})}, \boldsymbol{y}^{(n_a,n_{u_{n_a}})})\right\}
	\end{split}
	\end{equation*}
	where $n_{ij}$ is the number of images the $j$th user for the $i$th attribute labeled. The input feature is preprocessed such that
	\begin{equation}\label{A1}
	\begin{split}
	\sum\limits_{k=1}^{n_{ij}}\left(x^{(i,j)}_{k,l}\right)^2=1
	\end{split}
	\end{equation}, where $x^{(i,j)}_{k,l}$ is the $(k,l)$th entry of $\boldsymbol{X}^{(i,j)}$.\\
	\indent For binary attributes, $\boldsymbol{X}^{(i,j)} \in \mathbb{R}^{n_{ij} \times d}$
	is the corresponding feature matrix for the images that the $j$th user of the $i$th attribute labeled. Each row of $\boldsymbol{X}^{(i,j)}$ represents the low-level feature for a corresponding image.   $\boldsymbol{y^{(i,j)}} \in \{-1,1\}^{n_{ij}}$ is the corresponding label vector \footnote[1]{If not explained, $A^{(i,j)}$ denotes the corresponding variable of $A$ for the $j$th user of the $i$th attribute; $A^{(i)}$ denotes the corresponding variable for the $i$th attribute.}. If $y^{(i,j)}_k =1 $, then the user thinks that the corresponding attribute is present in the $k$th image, otherwise it will be labeled as -1.\\
	\indent For relative attributes, we need to solve a ranking problem. The corresponding users are given a set of image pairs $\{(\boldsymbol{x}^{(i,j)}_{1,k},\boldsymbol{x}^{(i,j)}_{2,k})\}_{k=1}^{n_{ij}}$. Since we adopt linear models in this paper, we define the $k$th row of $\boldsymbol{X}^{(i,j)}$ as $\boldsymbol{X}_k^{(i,j)} = \boldsymbol{x}^{(i,j)}_{1,k}- \boldsymbol{x}^{(i,j)}_{2,k}$. For the $k$th pair $y^{(i,j)}_k =1 $ if the user thinks that  the corresponding attribute has a stronger expression in the former image (\textit{say 1 is more open than 2}); $y^{(i,j)}_k =0$ if the user thinks that both images show similar strength for the current attribute (\textit{say 1 is as open as 2}); $y^{(i,j)}_k =-1 $ if the user thinks that  the corresponding attribute has a weaker expression in the former image (\textit{say 1 is less open than 2}). \\
	\indent As mentioned in the introduction section, our goal is  to learn a predictor $\boldsymbol{f}^{(i,j)}$ for each of the personalized label vectors $\boldsymbol{y}^{(i,j)}$. In this paper, we assume that $f^{(i,j)}(\cdot)$ has a linear form :
	\begin{equation}
	\begin{split}
	\boldsymbol{f}^{(i,j)} =\boldsymbol{X^{(i,j)}} \boldsymbol{w^{(i,j)}}
	\end{split}
	\end{equation}
	where $\boldsymbol{w}^{(i,j)}$ is the corresponding model weight. \\
	\indent Now we are ready to introduce the modeling of  $\boldsymbol{w}^{(i,j)}$.
	Note that our underlying problem could be comprehended in a common-to-special manner: there is a common pattern that captures the general cognition of a given object; an  attribute-specific pattern is also necessary to express the attribute-level common pattern; finally, a user-specific factor is necessary to describe the personalized preference. As a result, we adopt a three-component additive decomposition of $\boldsymbol{w}^{(i,j)}$ :
	\begin{equation}
	\begin{split}
	\boldsymbol{w}^{(i,j)} = \boldsymbol{\theta} + \boldsymbol{p}^{(i)} + \boldsymbol{\boldsymbol{u}}^{(i,j)}
	\end{split}
	\end{equation}
	The practical meaning of these three components could be explained as follows:
	\begin{itemize}
		\item $\boldsymbol{\theta}$(General Cognition Factor): $\boldsymbol{\theta} \in \mathbb{R}^{d\times1}$ is the  global factor that captures the overall cognition for the given class of object (\textit{say for shoes dataset, this factor captures the overall cognition about shoes}). $\theta$ is shared among all subtasks.
		\item $\boldsymbol{p}^{(i)}$(Attribute Specific Bias Factor): An  attribute-specific factor that captures the  bias of the $i$ th attribute with respect to the global cognition. For mathematical convenience, we denote $\boldsymbol{P}=[\boldsymbol{p}^{(1)},\cdots, \boldsymbol{p}^{(n_a)}]$, and we have $\boldsymbol{P}\in \mathbb{R}^{d\times n_a}$.
		\item $\boldsymbol{u}^{(i,j)}$ (User Specific Bias Factor): A user specific factor that captures the personal bias for the $j$th user of the $i$th attribute. In order to simplify the mathematical expressions, we define $\boldsymbol{U}^{(t)} = [\boldsymbol{u}^{(t,1)},\cdots,\boldsymbol{u}^{(t,n_{u_t})}]$,  $\boldsymbol{U} = [\boldsymbol{U}^{(1)},\cdots,\boldsymbol{U}^{(n_a)}]$,  thus we have $\boldsymbol{U} \in \mathbb{R}^{d \times n_u}$, where $n_u = \sum\limits_{i=1}^{n_a}n_{u_i}$.
	\end{itemize}
	Since we adopt a linear form for $\boldsymbol{f}^{(i,j)}$, it is natural to assume that the real response $\boldsymbol{y}^{(i,j)}$ could be interpreted as the true predictor in our proposed model :$\boldsymbol{f}^{*^{(i,j)}} = \boldsymbol{x}^{(i,j)}\boldsymbol{w}^{*^{(i,j)}}$  plus a Gaussian noise $\boldsymbol{\delta}^{(i,j)} \sim \mathcal{N}(\boldsymbol{0}, \sigma^2\boldsymbol{I})$: \\
	\begin{equation}\label{A3}
	\begin{split}
	\boldsymbol{y}^{(i,j)} =  \boldsymbol{f}^{*^{(i,j)}} + \boldsymbol{\delta}^{(i,j)}
	\end{split}
	\end{equation}
	where $\boldsymbol{w}^{*^{(i,j)}}= \boldsymbol{\theta}^* + \boldsymbol{p}^{*^{(i)}}+\boldsymbol{u}^{*^{(i,j)}}$.\\
	\indent As for the objective function, we adopt the least square loss as our empirical loss: $L(\boldsymbol{W})$ and a general regularizer 
	$\Omega(\cdot)$. We could thus formulate our problem as $(P1)$ as:
	\begin{equation*}
	(P1): \ \ \min\limits_{\boldsymbol{W}}   \underbrace{\sum\limits_{i=1}^{n_a}\sum\limits_{j=1}^{n_{u_i}} \norm{\boldsymbol{y}^{(i,j)}-\boldsymbol{X}^{(i,j)}\boldsymbol{w}^{(i,j)}}^2}_{L(\boldsymbol{W})} + \Omega(\boldsymbol{W})
	\end{equation*}
	where $\boldsymbol{W}:=\{\boldsymbol{w}^{(i,j)}\}_{(i,j)}$ is the set of all parameters.  \\
	\indent For our problem,  user annotated labels are often limited. Furthermore, the low level features for an image are located in a high dimensional space. Then it is necessary to encourage sparse models to reduce model complexity. To preserve the relationship among subtasks, we also need to leverage a shared feature subset.
	Consequently, we penalize $\boldsymbol{\theta}$ and $\boldsymbol{P}$ with $\ell_1$ norm and $\ell_{1,2}$ norm,  respectively. Meanwhile,  $\boldsymbol{u}^{(i,j)} \neq \boldsymbol{0}$ only when the corresponding user has a specific bias with respect to the popular opinion. We penalize $\boldsymbol{U}^\top$ with $\ell_{1,2}$ norm to leverage column-wise sparsity. Above all, $\Omega(\boldsymbol{W})$ could be represented as follows :
	\begin{equation*}
	\begin{split}
	\Omega(\boldsymbol{W}) = \lambda_1\norm{\boldsymbol{\theta}}_{1} + \lambda_2\norm{\boldsymbol{P}}_{1,2} + \lambda_3 \norm{\boldsymbol{U}^\top}_{1,2}
	\end{split}
	\end{equation*}
	\indent Putting all  together, $(P1)$ could be reformed as :
	\begin{equation}\label{finpro}
	\begin{split}
	\min\limits_{\boldsymbol{W}}   \sum\limits_{i=1}^{n_a}\sum\limits_{j=1}^{n_{u_i}} \norm{\boldsymbol{y}^{(i,j)}-\boldsymbol{X}^{(i,j)}(\boldsymbol{\theta} + \boldsymbol{p}^{(i)} + \boldsymbol{\boldsymbol{u}}^{(i,j)})}^2 \\
	+\lambda_1\norm{\boldsymbol{\theta}}_{1} + \lambda_2\norm{\boldsymbol{P}}_{1,2} + \lambda_3 \norm{\boldsymbol{U}^\top}_{1,2}
	\end{split}
	\end{equation}
	Figure 1  illustrates the expected structure of the three components in the proposed model. It could be seen that both the attribute level and the user level task correlations are included in our proposed model. \\
	\indent To end this section, we introduce two important mathematical properties of $(P1)$ as proposition 1 and proposition 2. Refer to our supplementary materials for a detailed proof of proposition 1 and proposition 2.
	\begin{prop}[Global Optimality]
		$P1$ is jointly convex with respect to $\boldsymbol{\theta}$, $\boldsymbol{P}$, $\boldsymbol{U}$
	\end{prop}

	\begin{prop}[Lipschitz Continuous Gradient]
		Given two arbitrary feasible solutions $W$ and  $W'$, we have :
		\begin{equation*}
		\norm{\nabla{L(\tilde{W})} - \nabla{L(\tilde{W}')}} \leq \rho  \norm{\tilde{W} - \tilde{W}' }
		\end{equation*}
		where: \\ $\rho =6n_u\sqrt{(n_u+n_a+1)}\max\limits_{i,j}{\left[\sigma_1\left(\boldsymbol{X}^{(i,j)}\right)\right]}^2
		$\\ $\tilde{W} = [vec(\boldsymbol{\theta})^\top , vec(\boldsymbol{P})^\top, vec(\boldsymbol{U})^\top]^\top$ \\
		$\tilde{W'} = [vec(\boldsymbol{\theta'})^\top , vec(\boldsymbol{P'})^\top, vec(\boldsymbol{U'})^\top]^\top$,
		$n_u = \sum\limits_{i=1}^{n_a}n_{u_i}$
	\end{prop}
	\subsection{Optimization Method}
	According to proposition 1 and proposition 2, $L(\boldsymbol{W})$ is a convex and smooth function while $\Omega(\boldsymbol{W})$ is a convex non-smooth function. Similar as the related literatures \cite{fista,robust1,robust2}, the accelerated proximal gradient method is employed to solve $(P1)$.\\
	\indent According to proposition 1, $L(\cdot)$ is differentiable with Lipschitz continuous gradient. According to the basic mathematical properties of Lipschitz continuous functions, at iteration step $k$, for any reference point $\boldsymbol{W}^{ref_k} = (\boldsymbol{\theta}^{ref_k}, \boldsymbol{P}^{ref_k},\boldsymbol{U}^{ref_k})$, $\exists \rho_k >0$, such that :
	
	\begin{flalign}\label{upper bound}
	\begin{split}
	& L(\boldsymbol{W}) \leq 
	L(\boldsymbol{W}^{ref_k}) +
	\left<\nabla_P L(\boldsymbol{W}^{ref_k}), \Delta \boldsymbol{P}\right> \\
	&	+ \left<\nabla_{\theta} L(\boldsymbol{W}^{ref_k}),  \Delta \boldsymbol{\theta} \right> +
	\left<\nabla_U L(\boldsymbol{W}^{ref_k}),  \Delta \boldsymbol{W}\right>\\
	&+\dfrac{\rho_k}{2}\norm{\Delta\boldsymbol{P}}_F^2 + \dfrac{\rho_k}{2}\norm{\Delta\boldsymbol{\theta}}^2_2 + \dfrac{\rho_k}{2}\norm{\Delta\boldsymbol{U}}^2_F\\
	& \overset{def}{=} \hat{L}_{\boldsymbol{W}^{ref_k},\rho}(\boldsymbol{W})
	\end{split}
	\end{flalign}
	
	Following the Majorization-Minimization (MM) \cite{MM} scheme, at the k-th iteration, we could then solve $(P2)$ instead of updating the weights based on original problem:
	\begin{equation*}\label{P2}
	\begin{split}
	(P2) :
	(\boldsymbol{\theta}^k,\boldsymbol{P}^k, \boldsymbol{U}^k)  := \argmin_{\boldsymbol{\theta},\boldsymbol{P},
		\boldsymbol{U}} \hat{L}_{\boldsymbol{W}^{ref_k},\rho_k}(\boldsymbol{W}) + \Omega{(W)}
	\end{split}
	\end{equation*}
	It is obvious that $\boldsymbol{\theta},\boldsymbol{P}, \boldsymbol{U}$ are decoupled in $\hat{L}_{\boldsymbol{W}^{ref_k},\rho_k}(\boldsymbol{W})$. Hence, solving $(P2)$ is equivalent to solving the following three proximal subproblems simultaneously:
	\begin{equation}\label{P2}
	\begin{split}
	\boldsymbol{\theta}^k  := \argmin_{\boldsymbol{\theta}} \dfrac{1}{2} \left\norm{\boldsymbol{\theta} -\tilde{\boldsymbol{\theta}^k} \right}^2 + \frac{\lambda_1}{\rho_k} \norm{\boldsymbol{\theta}}_1
	\end{split}
	\end{equation}
	\begin{equation}\label{P2}
	\begin{split}
	\boldsymbol{P}^k  := \argmin_{\boldsymbol{P}} \dfrac{1}{2}\left\norm{\boldsymbol{P} - \tilde{\boldsymbol{P}}^k\right}^2 + \frac{\lambda_2}{\rho_k} \norm{\boldsymbol{P}}_{1,2}
	\end{split}
	\end{equation}	
	\begin{equation}\label{P2}
	\begin{split}
	\boldsymbol{U}^{k}  := \argmin_{\boldsymbol{U}} \dfrac{1}{2}\left\norm{\boldsymbol{U} - \tilde{\boldsymbol{U}}^{k}\right}^2 + \frac{\lambda_3}{\rho_k} \norm{\boldsymbol{U}^\top}_{1,2}\\ 
	\end{split}
	\end{equation}	
	where
	\begin{equation*}
	\begin{split}
	\tilde{\boldsymbol{\theta}}^k = \boldsymbol{\theta}^{ref_k} - \dfrac{1}{\rho_k}\nabla_{\theta}L(\boldsymbol{W}^{ref_k})
	\end{split}
	\end{equation*}
	
	\begin{equation*}
	\begin{split}
	\tilde{\boldsymbol{P}}^k=\boldsymbol{P}^{ref_k} - \dfrac{1}{\rho_k}\nabla_{P}L(\boldsymbol{W}^{ref_k})
	\end{split}
	\end{equation*}
	\begin{equation*}
	\begin{split}
	\tilde{\boldsymbol{U}}^{k}=\boldsymbol{U}^{ref_k} - \dfrac{1}{\rho_k}\nabla_{\boldsymbol{U}}L(\boldsymbol{W}^{ref_k})
	\end{split}
	\end{equation*}
	\begin{algorithm}[ht]           
		\caption{The accelerated proximal gradient method for solving $(P1)$}
		\KwIn{$\mathcal{T}$, $\lambda_1$,$\lambda_2$,$\lambda_3$,$\rho_0>0$ ,$\eta >1$}
		
		\KwOut{$\boldsymbol{\theta}$,$\boldsymbol{P}$,$\boldsymbol{U}$}

		Initialize $\boldsymbol{\theta}^0$,$\boldsymbol{P}^0$,$\boldsymbol{U}^0$\;
		$\boldsymbol{\theta}^{ref} :=\boldsymbol{\theta}^{0} $, 
		$\boldsymbol{P}^{ref} :=\boldsymbol{P}^{0} $, 
		$\boldsymbol{U}^{ref} :=\boldsymbol{U}^{0} $,$t_1:=1$,$k:=1$\;
		\While{Not Converged}{
			Solve $\boldsymbol{W}^k = (\boldsymbol{\theta}^k,\boldsymbol{P}^k,\boldsymbol{U}^k)$ with Eq.(\ref{solv1})-Eq.(\ref{solv3})\;
			Find the smallest $i_k$ such that when $\tilde{\rho} = \eta^{i_k}\tilde{\rho}_{k-1} $ : $L(\boldsymbol{W^k}) \leq \hat{L}_{\boldsymbol{W}^{ref_k},\tilde{\rho} }(\boldsymbol{W}^k)$ \;
			$\boldsymbol{\rho}_k :=\tilde{\rho}$\;
			update $\boldsymbol{W}^k$ again\; 
			$t_{k+1} = \dfrac{1+\sqrt{1+4t_k^2}}{2}$\;
			$dt:= \dfrac{t_k-1}{t_{k+1}}$\;
			$\boldsymbol{\theta}^{ref}:=
			\boldsymbol{\theta}^k +dt(\boldsymbol{\theta}^k -\boldsymbol{\theta}^{k-1} ) $\;
			$\boldsymbol{P}^{ref}:=
			\boldsymbol{P}^k +dt(\boldsymbol{P}^k -\boldsymbol{P}^{k-1} ) $\;
			$\boldsymbol{U}^{ref}:=
			\boldsymbol{U}^k +dt(\boldsymbol{U}^k -\boldsymbol{U}^{k-1} ) $\;
			$k:=k+1$\;
		}
	\end{algorithm} 
	
	\noindent All of these subproblems admit closed-form solutions :
	\begin{equation}\label{solv1}
	\begin{split}
	{\theta}^k_i  := sign(\tilde{\theta}^k_i)\left(\left|\tilde{\theta}^k_i\right|-\dfrac{\lambda_1}{\rho_k}\right)_+
	\end{split}
	\end{equation}
	\begin{equation}\label{solv2}
	\begin{split}
	\boldsymbol{P}_{i}^k := \left(1-\dfrac{\lambda2}{\rho_k \norm{\tilde{\boldsymbol{P}}_i^{k}}}\right)_+\tilde{\boldsymbol{P}}_i^{k}
	\end{split}
	\end{equation}	
	\begin{equation}\label{solv3}
	\begin{split}
	\left(\boldsymbol{u}^{(i,j)}\right)^k :=
	\left(1-\dfrac{\lambda3}{\rho_k \left\norm{	\left(\tilde{\boldsymbol{u}}^{(i,j)}\right)^k\right}}\right)_+\left(\tilde{\boldsymbol{u}}^{(i,j)}\right)^k
	\end{split}
	\end{equation}
	Furthermore, the nestrov's acceleration strategy is employed to update the reference point $\boldsymbol{W}^{ref}$. Integrating all the results, an efficient algorithm to solve $(P1)$ is introduced as Algorithm 1. According to the theoretical analyses proposed in \cite{fista}, $\boldsymbol{W}^k$ could converge to a global optimal solution with rate $\mathcal{O}(\dfrac{1}{k^2})$, which is the provable optimal rate for first-order methods.
	\subsection{Theoretical Analysis}
	Following \cite{robust2}, we will propose the performance bound of our algorithm based on assumption 1.\\
	\indent Here we define a set $\mathcal{N}(\boldsymbol{A})$ for a matrix (vector)  $\boldsymbol{A}$ as the indexes for zero rows (entries): $\mathcal{N}(\boldsymbol{A}) = \{i: \boldsymbol{A}_i=\boldsymbol{0}\}$, and $\mathcal{N}_\perp(A)$ as the complement of $\mathcal{N}(\boldsymbol{A})$: $\mathcal{N}_\perp(A) =  \{i: \boldsymbol{A}_i \neq \boldsymbol{0}\}$. $\left|\mathcal{N}(\boldsymbol{A})\right|$ is the number of zero rows (entries) of a matrix(vector) $\boldsymbol{A}$, and we have similar definition for $\left|\mathcal{N}_\perp(\boldsymbol{A})\right|$.
	Now we provide the basic assumption of the main result as assumption 1.
	\begin{asm}
		It is defined that  : $\boldsymbol{X}\overset{def}{=} \oplus_{i,j}\boldsymbol{X}^{(i,j)}$,\\ \[\bar{\boldsymbol{W}} \overset{def}{=} [\boldsymbol{w}^{(1,1)^\top},  \boldsymbol{w}^{(1,2)^\top}, \cdots , \boldsymbol{w}^{(n_a,n_{u{n_a}})^\top} ]^\top\]
		Let $ 0 \leq n_{\theta} \leq d $ be the upper bound of $|\mathcal{N}_{\perp}(\boldsymbol{\theta^*})|$, $ 0\leq n_{p} \leq d$ be the upper bound of $|\mathcal{N}_{\perp}(\boldsymbol{P}^*)|$, $0 \leq n_{u,a} \leq n_{u} $
		be the upper bound of $|\mathcal{N}_{\perp}(\boldsymbol{U^{*^{^\top}}})|$. We  assume that there exists nonnegative $\kappa_\theta$, $\kappa_p$, $\kappa_{u,a}$: \\
		\[\kappa_\theta = \min\limits_{\Gamma_{\theta},\Gamma_P,\Gamma_U} \frac{\left\norm{X\Gamma_{\overline{W}} \right}}{\sqrt{n_{min}n_u}\norm{\Gamma_{\theta}^{\mathcal{N}_{\perp}(\theta)}}_2} \]
		\[\kappa_p = \min\limits_{\Gamma_{\theta},\Gamma_P,\Gamma_U} \dfrac{\left\norm{X\Gamma_{\overline{W}} \right}}{\sqrt{n_{min}n_u}\norm{\Gamma_{\boldsymbol{P}}^{\mathcal{N}_{\perp}(P)}}_F} \]
		\[\kappa_{u,a} = \min\limits_{\Gamma_{\theta},\Gamma_P,\Gamma_U} \dfrac{\left\norm{X\Gamma_{\overline{W}} \right}}{\sqrt{n_{min}n_u}\norm{(\Gamma_{U})^{\top^{\mathcal{N}_{\perp}(U^{\top})}}}_F} \]
		where $\Gamma_\theta \in \mathbb{R}^{d}$ is a function of $\boldsymbol{\theta}$; $\Gamma_P \in \mathbb{R}^{d \times n_a}$ is a function of $\boldsymbol{P}$; $\Gamma_{U} \in \mathbb{R}^{d \times n_u}$ is a function of $\boldsymbol{U}$. Define $\Gamma_{w^{(i,j)}} = \Gamma_{\theta}　+ \Gamma_{p^{(i)}} + \Gamma_{u^{(i,j)}}$, then :
		$\Gamma_{\overline{W}}= [\Gamma_{w^{(1,1)}}^\top,\Gamma_{w^{(1,2)}}^\top,\cdots,\Gamma_{w^{(n_a,n_{u_{n_a}})}}^\top]^\top$.
		Furthermore, it is assumed that the following inequalities hold :
		$\norm{\Gamma_{\boldsymbol{\theta}}^{\mathcal{N}(\theta)}}_{1} \leq \beta_\theta\norm{\Gamma_{\boldsymbol{\theta}}^{\mathcal{N}_{\perp}(\theta)}}_{1}$,
		$\norm{\Gamma_{\boldsymbol{P}}^{\mathcal{N}(P)}}_{1,2} \leq \beta_p\norm{\Gamma_{\boldsymbol{P}}^{\mathcal{N}_{\perp}(P)}}_{1,2}$,  $\norm{\Gamma_{U}^\top{^{\mathcal{N}(U^{\top})}}}_{1,2} \leq \beta_{u,a}\norm{{\Gamma_{U}^\top}^{\mathcal{N}_{\perp}(U^{\top})}}_{1,2}$.
		where $\beta_\theta$, $\beta_p$, $\beta_{u,a}$ are positive scalars.
	\end{asm}
	Note that the assumption on $\kappa_{\theta}$, $\kappa_{p}$, $\kappa_{u,a}$ is based on the restricted eigenvalue assumption \cite{assm1}, which has been widely used in existing multi-task literatures \cite{robust1,robust2}.\\
	\indent According to the notations in assumption 1, the squared error between the predicted value $\boldsymbol{x}^\top\boldsymbol{w}$ and the real value $\boldsymbol{f}$ could be formed as :
	$\norm{\boldsymbol{X}\boldsymbol{\overline{W}}- \boldsymbol{F}}$,  where $\boldsymbol{F}$ is defined as :
	\[\boldsymbol{F} = [\boldsymbol{f}^{*^{(1,1)^\top}}, \boldsymbol{f}^{*{(1,2)^\top}}, \cdots,\boldsymbol{f}^{*{(n_a,n{u_{n_a}})^\top}} ]^\top\]
	Let $\hat{\boldsymbol{W}} =(\hat{\boldsymbol{\theta}},\hat{\boldsymbol{P}},\hat{\boldsymbol{U}})$ be an optimal solution of $(P1)$. According to Eq.(\ref{A3}), we define $\boldsymbol{W}^* =(\boldsymbol{\theta}^*,\boldsymbol{P}^*, \boldsymbol{U}^*)$. Our main result could be presented as Theorem 1.
	\begin{thm}[Performance Bounds] Define $\alpha = 2\sigma\sqrt{dn_u+t}$, choose $\lambda_1$, $\lambda_2$, $\lambda_3 $ as : $\lambda_1 \ge {n_u}\alpha$,  $\lambda_2 \ge
		\tilde{n}\alpha$, $\lambda_3 \ge \alpha$
		, where $n_u = \sum\limits_{i=1}^{n_{a}}n_{u_i}$ and $\tilde{n} = \sqrt{\sum\limits_{i}n_{u_i}^2}$.
		Given Assumption 1, let
		\begin{equation*}
		\begin{split}
		\zeta = \dfrac{\lambda_1\sqrt{n_{\theta}}}{\kappa_{\theta}} + \dfrac{\lambda_2\sqrt{n_{p}}}{\kappa_{p}} + \dfrac{\lambda_2\sqrt{n_{u,a}}}{\kappa_{u,a}}
		\end{split}
		\end{equation*}
		for $t >0$, we have :
		\begin{equation}\label{thm1perf}
		\begin{split}
		\mathbb{P}\left(\frac{1}{n_{min}n_u}\norm{\boldsymbol{X}\boldsymbol{\overline{W}- \boldsymbol{F}}}^2 \leq (\dfrac{2\zeta}{n_{min}n_u})^2\right) \ge \delta(t)
		\end{split}
		\end{equation}
		\begin{equation}\label{thm1theta}
		\begin{split}
		\mathbb{P}\left(\norm{\boldsymbol{\hat{\boldsymbol{\theta}}} - \boldsymbol{\boldsymbol{\theta^*}}}_{1} \leq \dfrac{2(\beta_\theta+1)\sqrt{n_\theta}}{\kappa_\theta{n_{min}n_u}}\zeta\right) \ge
		\delta(t)
		\end{split}
		\end{equation}
		\begin{equation}\label{thm1p}
		\begin{split}
		\mathbb{P}\left(\norm{\boldsymbol{\hat{\boldsymbol{P}}} - \boldsymbol{\boldsymbol{P^*}}}_{1,2} \leq \dfrac{2(\beta_p+1)\sqrt{n_p}}{\kappa_p{n_{min}n_u}}\zeta\right) \ge \delta(t)
		\end{split}
		\end{equation}
		\begin{equation}\label{thm1u}
		\begin{split}
		\mathbb{P}\left(\norm{\boldsymbol{\hat{\boldsymbol{U}}^{\top}} - \boldsymbol{\boldsymbol{U^{*^\top}}}}_{1,2} \leq \dfrac{2(\beta_{u,a}+1)\sqrt{n_{u,a}}}{\kappa_{u,a}{n_{min}n_u}}\zeta \right) \ge \delta(t)
		\end{split}
		\end{equation}
		where $\delta(t) = 1 - \dfrac{1}{\sqrt{2\pi Z_{dn_u}(t)}}exp\left(-\dfrac{Z_{dn_u}(t)}{2}\right)$; $n_{min} = \min\limits_{i,j}n_{ij}$ and $Z_{dn_u}(t) = t-dn_ulog(1+\frac{t}{dn_u})$.
	\end{thm}
	
	\begin{rem}
		According to theorem 1, if \[\zeta= o\Bigg(n_{min}n_u\min\left\{1,\dfrac{\kappa_{\theta}}{\sqrt{n_\theta}},\dfrac{\kappa_{p}}{\sqrt{n_p}},\dfrac{\kappa_{u,a}}{\sqrt{n_{u,a}}}\right\}\Bigg)\] We have : $\mathbb{E}(\boldsymbol{X}^{(i,j)}\boldsymbol{\hat{w}}^{(i,j)} -\boldsymbol{f}^{*^{(i,j)}}) \rightarrow 0$, $\boldsymbol{\hat{\theta}} \overset{\ell_1}{\rightarrow} \boldsymbol{\theta^*}$, $\boldsymbol{\hat{P}} \overset{\ell_{1,2}}{\rightarrow} \boldsymbol{P^*}$,
		and $\boldsymbol{\hat{U}^\top} \overset{\ell_{1,2}}{\rightarrow} \boldsymbol{U}^{*^\top}$ hold with high probability when $n_{min} \rightarrow \infty$. Furthermore, though the proof of theorem 1 uses standard techniques developed in \cite{robust2}, $\delta(t)$ in theorem 1 is a tighter probability bound than $\left(1-exp(-\dfrac{Z_{dn_u}(t)}{2})\right)$ proposed in \cite{robust2} and
		$\left(1-n_uexp(-\dfrac{Z_{dn_u}(t)}{2})\right)$ proposed in \cite{robust1}, for sufficiently large $t$.
	\end{rem}
	According to theorem 1 and remark 1, we see that our proposed method could both leverage good performance and estimate the parameters well with high probability.  
	\section{Experiment}
	Now in this section, we show our experiment results on a simulated dataset, and two real world datasets respectively.
	\subsection{Experiment Setting}
	For each subtask, we randomly split the corresponding samples into a training subset and test subset, with 40\% and 80\% of the samples selected as traning set respectively. For each involved algorithm, the hyper-parameters are tuned based on a 3 fold cross validation on the training set, and the average performance of the test set on 5 different splits are recorded. It is important to note that, different with the setting of \cite{user1}, we will not use any extra dataset for pre-training in this paper. Furthermore, the training data is preprocessed according to Eq.(\ref{A1})
	\subsection{Simulated Dataset}
	
	\subsubsection{Data Generation}
	For simulated dataset, our goal is to verify that the proposed algorithm could reach reasonable performance based on our theoretical analysis. We here define a regression problem for this dataset.
	To this end, we generate simulated features and continuous user scores (but not discrete labels) for 5 attributes, and all $n_{u_i}$s are fixed as 10.   Furthermore, we set the dimensionality $d$ as 50. For the $(i,j)$th subtask, 300 samples are generated such that $\boldsymbol{X}^{(i,j)} \sim \mathcal{N}(0,4\boldsymbol{I}) $ and $\boldsymbol{y}^{(i,j)} = \boldsymbol{X}^{(i,j)}\boldsymbol{w}^{(i,j)} + \mathcal{N}(0,\boldsymbol{I})$. To leverage group sparsity of $\boldsymbol{W}$: $\boldsymbol{\theta}$ is generated as $\boldsymbol{\theta} \sim \mathcal{N}(\boldsymbol{1},4*\boldsymbol{1})$, and the first 15 elements are set as zero; $\boldsymbol{P}$ is generated as $\boldsymbol{P} \sim \mathcal{N}(\boldsymbol{1}, 5\boldsymbol{I})$ and the 20-35 th rows of $\boldsymbol{P}$ are set as $\boldsymbol{0}$;  $\boldsymbol{U}$ is generated as $\boldsymbol{U} \sim \mathcal{N}(\boldsymbol{1}, 10\boldsymbol{I})$ and the first 2 columns of each $\boldsymbol{U}^{(t)}$ are set as $\boldsymbol{0}$.
	\subsubsection{Competitors and Evaluation Metric}
	To verify the effectiveness of our proposed algorithm, we compare our algorithm with four benchmark algorithms for regression: Support Vector Regression (SVR), Lasso regression (Lasso), Ridge regression (Ridge), and the Elastic Net. Meanwhile, all of these four benchmark algorithms are employed in a user-exclusive manner : the predictors for all the subtasks are trained separately as independent tasks. To evaluate the generalized performance of the algorithms, for a given attribute, we adopt $\overline{NMSE}$, the average value of normalized mean square error (NMSE) on all users for a given attribute, as the evaluation metric.
	
	\begin{table}[h]
		\centering
		\caption{Performance comparison for the simulated dataset with 40\% samples selected  as training data}
		\begin{tabular}{ccccc}
			\toprule
			SVR   & Ridge & Lasso & Elastic Net & ours \\
			\midrule
			1.000  & 0.830  & 2.99E-05 & 9.96E-05 & \textbf{2.70E-05} \\
			1.000  & 0.830  & 3.20E-05 & 9.92E-05 & \textbf{2.82E-05} \\
			1.000  & 0.830  & 3.15E-05 & 1.07E-04 & \textbf{2.78E-05} \\
			0.998  & 0.829  & 4.20E-05 & 1.14E-04 & \textbf{3.75E-05} \\
			1.005  & 0.834  & 2.88E-05 & 9.34E-05 & \textbf{2.72E-05} \\
			\midrule
			1.001  & 0.830  & 3.28E-05 & 1.03E-04 & \textbf{2.95E-05} \\
			\bottomrule
		\end{tabular}%
	\end{table}%
	
	\begin{table}[htbp]
		
		\centering
		\caption{Performance comparison for the simulated dataset with 80\% samples selected  as training data}
		\begin{tabular}{rrrrr}
			\toprule
			\multicolumn{1}{c}{SVR} & \multicolumn{1}{c}{Ridge} & \multicolumn{1}{c}{Lasso} & \multicolumn{1}{c}{Elastic Net} & \multicolumn{1}{c}{ours} \\
			\midrule
			1.005  & 0.884  & 2.27E-05 & 5.19E-05 & \textbf{2.11E-05} \\
			1.009  & 0.888  & 2.37E-05 & 5.27E-05 & \textbf{2.17E-05} \\
			1.009  & 0.887  & 2.30E-05 & 5.14E-05 & \textbf{2.12E-05} \\
			1.010  & 0.889  & 2.79E-05 & 5.64E-05 & \textbf{2.60E-05} \\
			1.019  & 0.896  & 2.03E-05 & 4.79E-05 & \textbf{1.89E-05} \\
			\midrule
			1.010  & 0.889  & 2.35E-05 & 5.20E-05 & \textbf{2.18E-05} \\
			\bottomrule
		\end{tabular}%
	\end{table}%
	\subsubsection{Performance Comparison}
	According to Table 1 and Table 2, we could draw the following conclusions. The six rows in these two tables records the performance for attributes 1-5 and their average, respectively.  On one hand, due to the inability to leverage sparse parameters, we see that performance of SVR and Ridge couldn't outperform the other three algorithms. On the other hand, our proposed algorithm reaches the best performance based on $\overline{NMSE}$, which verifies the effectiveness of our proposed algorithm. T
	\subsubsection{Parameters Recovery}
	Now we show the ability of our algorithm to recover the structured parameters. With the same simulated dataset, we select 80\% the samples as training data. According to  theorem 1, we set $t=10$, $\lambda_1 = 2n_u\alpha $, $\lambda_2 = 2.5 \tilde{n}\alpha$, $\lambda_3= 32\alpha$. Figure 2 shows the resulting parameters of our algorithm. Note that we extend $\theta$ and $P$ to $\mathbb{R}^{d \times n_u}$, so that they could match the size of
	$U$. As is shown in this figure, we conclude that all of these three sets of parameters could roughly reach their expected structure.
	\begin{figure}[ht]\label{weights}
		\centering
		\subfigure[$\theta$]{
			\includegraphics[scale=0.2]{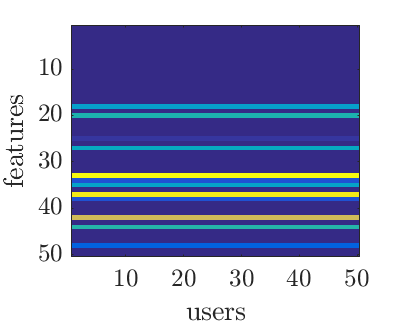}
		}
		\subfigure[$P$]{
			\includegraphics[scale=0.2]{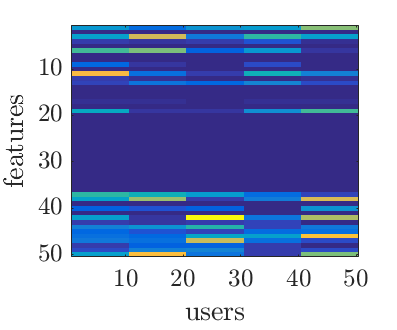}
		}
		\subfigure[$U$]{
			\includegraphics[scale= 0.2]{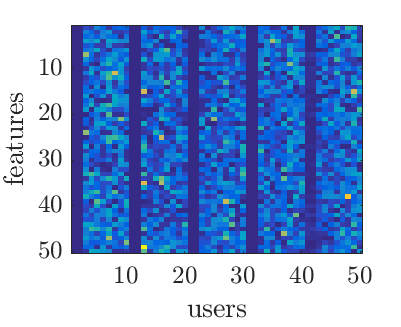}
		}
		\caption{Figures for the magnitude of $\theta$,$P$ and $U$. For comparison convenience, both $\theta$ and $P$ are extended to $\mathbb{R}^{d \times n_u}$ matrix}
	\end{figure}
	\begin{figure}[!htb]\label{opt}
		\centering
		\subfigure[40\%]{
			\includegraphics[scale=0.38]{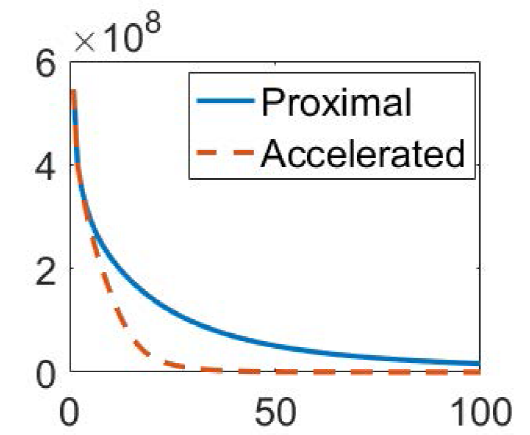}
		}
		\subfigure[80\%]{
			\includegraphics[scale=0.38]{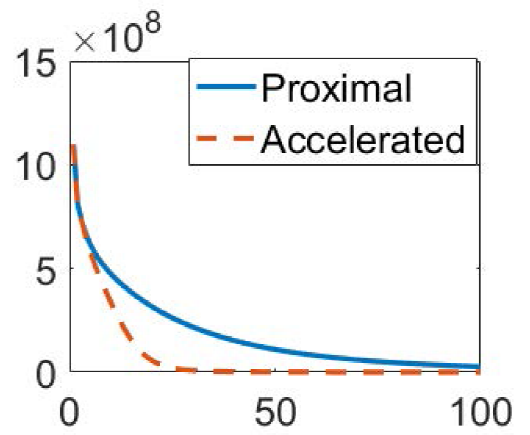}
		}
		\caption{Effect of nestrov's Acceleration with (a) 40\% random selected samples as training data and (b) 80\% samples as training data. The y axis represents the average loss function of $(P1)$ on 10 repetitions, and the x axis represents the iteration number.   }
	\end{figure}
	\begin{table}[ht]\label{resbin80}
		\centering
		\caption{Performance comparison for the binary attribute dataset}
		\begin{tabular}{lcc}
			\bottomrule
			\multicolumn{1}{l}{\multirow{2}[3]{*}{algorithm}} & \multicolumn{2}{c}{accuracy} \\
			\cmidrule{2-3}               & 40\% & 80\% \\
			\midrule
			SVM  & 0.6278  & 0.6756  \\
			MLP  & 0.5057  & 0.4997  \\
			\midrule
			user exclusive & 0.6549  & 0.6913  \\
			user adaptive & 0.6771  & 0.6956  \\
			\midrule
			rMTFL-G & 0.6421  & 0.6941  \\
			rMTFL-U & 0.6659  & 0.7056  \\
			\midrule
			ours & \textbf{0.6894} & \textbf{0.7121} \\
			\bottomrule
		\end{tabular}%
	\end{table}%
	\subsubsection{Effect of Nestrov's Acceleration Strategy}
	Next, we conduct empirical study the on the effect of Nestrov's Acceleration. To do this, we randomly select 40\% and 80\%
	samples as training data respectively. For each ratio, we run our algorithm 10 times for 100 iterations with different initial parameters. The average loss function per iteration both with and without the nestorv's acceleration strategy are presented in Figure 3. According to this figure, we conclude that, for both ratios, the accelerated algorithm starts to converge before the 50th iteration, which is much earlier than that of the ordinary proximal algorithm. We thus draw the conclusion that this acceleration strategy could successfully leverage faster convergence of our algorithm.
	
	
	\begin{table}[!hb]
		\centering
		\caption{Performance comparison for the relative attribute dataset}
		\begin{tabular}{lcc}
			\toprule
			\multicolumn{1}{l}{\multirow{2}[4]{*}{algorithm}} & \multicolumn{2}{c}{accuracy} \\
			\cmidrule{2-3}            & 40\% & 80\% \\
			\midrule
			rel\_attr & 0.4797  & 0.5195  \\
			RankNet & 0.4791  & 0.4721  \\
			RankBoost & 0.4669  & 0.5251  \\
			\midrule
			user exclusive & 0.4753  & 0.5303  \\
			user adaptive & 0.4777  & 0.5336  \\
			\midrule
			rMTFL-G & 0.4807  & 0.5074  \\
			rMTFL-U & 0.4838  & 0.5433  \\
			\midrule
			ours & \textbf{0.5119} & \textbf{0.5546} \\
			\bottomrule
		\end{tabular}%
		\label{resrel40}
	\end{table}%
	
	\subsection{Shoes Dataset with Binary Attributes}
	\subsubsection{Dataset Description}
	For attribute learning, we use the shoes Dataset \cite{user1,whittle} which contains 14,658 online shopping images. Here we choose 6 user labeled attributes from the original dataset: bright, ornate, shiny, high, long, formal. For each of the attribute, user specific binary labels were collected on 60 images from 10 workers \cite{user1}. In other words, we have $n_a =6$, $n_{u_i}=10$ and $n_{ij} = 60$. To form the feature of the images, we concatenate the GIST and color histograms provided by the original dataset.
	\subsubsection{Competitors and Evaluation Metric}
	To show the effectiveness of our proposed algorithm on binary personal attribute learning, we compare our algorithm with three kinds of algorithms: global algorithms, user-specific algorithms and multi-task algorithms. For global algorithms, during the training phase, user specific labels are first processed to global labels via majority voting. For valid set and test set the user specific labels are used directly for performance evaluation. This kind of algorithms include: the Support Vector Machine(SVM), and Multi Layer Perceptron with single hidden layer (MLP).  For user specific algorithms, we employ the user exclusive method where one SVM classifier is trained for each user independently and the user adaptive method proposed in \cite{user1}. For multi-task algorithms, we employ rMTFL \cite{robust2}, which shares similar model assumption as the proposed model, as the benchmark, and both the global version (rMTFL-G) and the user specific version (rMTFL-U) are considered. The setting of  rMTFL-G is the same as that of the other global algorithms, except that the global classifiers for different attributes are trained in a multi-task manner. While for rMTFL-U, we regard all the user-specific classifiers as subtasks of rMTFL.  To evaluate the generalized performance of different algorithms, the average value of the classification accuracy among all users is adopted as the performance metric.
	\subsubsection{Results}
	Table 3 shows the average performance on 5 splits for all these algorithms when 40\% and 80\% of the samples are chosen as training data respectively. We could observe that our proposed algorithm reaches the best average accuracy.
	\subsection{Shoes Dataset with Relative Attributes}
	\subsubsection{Dataset Description} Here we use the same shoes dataset as the binary attribute experiment. The only difference is that the user specific labels are collected based on  relative attribute between a pair of shoes images. For this task, we choose 6 attributes : pointy, bright, ornate, shiny, sporty and feminine. Similar as the previous subsection, we have $n_a =6$, $n_{u_i}=10$, $n_{ij}=60$.
	\subsubsection{Competitors and Evaluation Metric}
	For relative attribute learning, we also adopt the aforementioned three kinds of algorithms as benchmarks. The global models include: ranksvm models for relative attribute learning \cite{rel} (rel\_attr), ranknet \cite{ranknet}, and rankboost \cite{rankboost}. For the user exclusive model we train one rel\_attr model for each user independently. And user adaptive model is the same as \cite{user1} except that we do not use any extra data for pre-training. rMTFL-G and rMTFL-U
	are the same as that used in the previous subsection except that the input feature for an image pair is processed as what mentioned in the ``Model Formulation'' section to fit these algorithms to ranking problems. We use the average ranking accuracy among all users for all tasks as our evaluation metric. 
	\subsubsection{Results}
	Table 4  show the performance comparison when 40\% and 80\% of the samples are chosen as training data respectively. 
	It is concluded that our proposed algorithm could reach reasonable improvements with respect to the benchmarks,  which demonstrate its effectiveness.  
	\section{Conclusion}
	In this paper, we propose a hierarchical multi-task model for user specific attribute learning across multiple attributes with a common-to-special decomposition of the model weights. Specifically, our model weights include a common cognition factor, an attribute-specific factor and a user specific factor. The well-known accelerated proximal gradient method is employed to solve this model. Based on assumption 1, we prove theoretically that the proposed algorithm
	could both leverage good performance and estimate the true parameters well with high probability. The experiment results further verify the effectiveness of our proposed model.
	\section{Acknowledgments}
	The research of Zhiyong Yang and Qingming Huang was supported in part by National Natural Science Foundation of China: 61332016, U1636214, 61650202 and 61620106009, in part by Key Research Program of Frontier Sciences, CAS: QYZDJ-SSW-SYS013. The research of Qianqian Xu was supported in part by National Natural Science Foundation of China (No. 61672514, 61390514, 61572042), CCF-Tencent Open Research Fun. The research of Xiaochun Cao was Supported by National Key Research and Development Plan (No.2016YFB0800603), National Natural Science Foundation of China (No.U1636214, 61422213).
	\bibliography{aaai}
	\bibliographystyle{aaai}
\section{Appendix}
\subsection{Proof of Propositions}
\begin{Pf1}
	The proof directly follows the operations that preserves the convexity of a function \cite{cvx} \qed
\end{Pf1}

\begin{Pf2}
	
	Denote $\Delta_{ij}$ as :
	\begin{equation}
	\begin{split}
	\Delta_{ij} = 2(\boldsymbol{X}^{(i,j)\top}\boldsymbol{X}^{(i,j)}\boldsymbol{w}^{(i,j)} - \boldsymbol{X}^{(i,j)\top}\boldsymbol{y}^{(i,j)}
	)\end{split}
	\end{equation}
	We could reach :
	\begin{align*}
	&\nabla_{\boldsymbol{\theta}}L(\boldsymbol{W}) =
	\sum\limits_{i=1}^{n_a}\sum\limits_{j=1}^{n_{u_i}}\Delta_{i,j}\\
	&\nabla_{\boldsymbol{p}^{(i)}}L(\boldsymbol{W}) =
	\sum\limits_{j=1}^{n_{u_i}}\Delta_{i,j}\\
	&\nabla_{\boldsymbol{u}^{(i,j)}}L(\boldsymbol{W}) =
	\Delta_{i,j}
	\end{align*}
	
	We further denote 
	\begin{equation*}
	\begin{split}
	&dL_{i,j} = 2\boldsymbol{X}^{(i,j)\top}\boldsymbol{X}^{(i,j)}(\boldsymbol{w}^{(i,j)} -\boldsymbol{w'}^{(i,j)})\\
	&d\boldsymbol{\theta} = \nabla_{\boldsymbol{\theta}}L(\boldsymbol{W}) -\nabla_{\boldsymbol{\theta'}}L(\boldsymbol{W'})\\ &d\boldsymbol{p}^{(i)} = \nabla_{\boldsymbol{p}^{(i)}}L(\boldsymbol{W}) -\nabla_{\boldsymbol{p'}^{(i)}}L(\boldsymbol{W'})\\ &d\boldsymbol{u}^{(i,j)} = \nabla_{\boldsymbol{u}^{(i,j)}}L(\boldsymbol{W}) -\nabla_{\boldsymbol{u'}^{(i,j)}}L(\boldsymbol{W'})
	\end{split}
	\end{equation*}
	
	It thus follows that 
	\begin{equation}
	\begin{split}
	&\norm{\nabla{L(\tilde{W})} - \nabla{L(\tilde{W}')}} \\
	&=\sqrt{\norm{d\boldsymbol{\theta}}^2+\sum\limits_{i}\norm{d\boldsymbol{p}^{(i)}}^2+\sum\limits_{i}\sum\limits_{j}\norm{d\boldsymbol{u}^{(i,j)}}^2}\\
	&\overset{(\uppercase\expandafter{\romannumeral1})}{\leq}\norm{d\boldsymbol{\theta}}+\sum\limits_{i}\norm{d\boldsymbol{p}^{(i)}}+\sum\limits_{i}\sum\limits_{j}\norm{d\boldsymbol{u}^{(i,j)}}\\
	&= \norm{\sum\limits_{i=1}^{n_a}\sum\limits_{j=1}^{n_{u_i}}dL_{i,j}}+\sum\limits_{i=1}^{n_a}(\norm{\sum\limits_{j=1}^{n_{u_i}}dL_{i,j}}) + \sum\limits_{i=1}^{n_a}\sum\limits_{j=1}^{n_{u_i}}\norm{dL_{ij}}
	\\
	&\leq 6C\sum\limits_{i=1}^{n_a}\sum\limits_{j=1}^{n_{u_i}}\norm{(\boldsymbol{w}^{(i,j)} -\boldsymbol{w'}^{(i,j)})}  \\
	&\leq 6C\Big(\sum\limits_{i}n_{u_i}\norm{\boldsymbol{p}^{(i)} -\boldsymbol{p'}^{(i)}} + n_u\norm{\boldsymbol{\theta} -\boldsymbol{\theta'}} \\
	&+ \sum\limits_{i,j}\norm{\boldsymbol{u}^{(i,j)} -\boldsymbol{u'}^{(i,j)}}\Big)\\
	&\overset{(\uppercase\expandafter{\romannumeral2})}{\leq}
	6Cn_u\sqrt{n_a+1+n_u}\norm{\tilde{W} - \tilde{W}' }
	\end{split}
	\end{equation}
	where $C = \max\limits_{i,j}[\sigma_1(\boldsymbol{X}^{(i,j)})]^2$; (\uppercase\expandafter{\romannumeral1}) follow that $\sqrt{\sum\limits_{i}\boldsymbol{a}_i} \leq \sum\limits_{i} \sqrt{\boldsymbol{a}_i}, \ \forall\boldsymbol{a} \in \mathbb{R}^n$; (\uppercase\expandafter{\romannumeral2}) is due to the fact that $\norm{\boldsymbol{a}}_1 \leq \sqrt{n}\norm{\boldsymbol{a}}_2, \ \forall \boldsymbol{a} \in \mathbb{R}^n $
	\qed
\end{Pf2}

\subsection{Proof of the main result}

We first proof three fundamental lemmas which are the precursors of our main result.
\begin{lem}
	$\forall \boldsymbol{A},\boldsymbol{\hat{A}} \in \mathbb{R}^{d \times m}$, $m \ge 1$, we have :
	\begin{equation*}
	\begin{split}
	\norm{\boldsymbol{\hat{A}} -\boldsymbol{A}}_\ell + \norm{\boldsymbol{{A}}}_\ell - \norm{\boldsymbol{
			\hat{A}}}_\ell \leq 2 \norm{(\boldsymbol{\hat{A}} -\boldsymbol{A})^{\mathcal{N}_{\perp}(A)}}_\ell
	\end{split}
	\end{equation*}
	holds for all norm $\norm{\cdot}_\ell$ such that $\forall$
	$\mathcal{C} \subset [d]$ and $\mathcal{C}_\perp = [d]\backslash \mathcal{C}$ :\\
	\begin{equation}\label{assuml1}
	\begin{split}
	\norm{\boldsymbol{A}}_\ell = \norm{\boldsymbol{A}^{\mathcal{C}}}_\ell +  \norm{\boldsymbol{A}^{\mathcal{C}_\perp}}_\ell
	\end{split}
	\end{equation}
\end{lem}
\begin{proof}
	Since $\boldsymbol{A}^{\mathcal{N}(\boldsymbol{A})} = 0 $, we have :
	$\norm{(\boldsymbol{\hat{A}} -\boldsymbol{A})^{\mathcal{N}(A)}}_\ell = 
	\norm{(\boldsymbol{\hat{A}})^{\mathcal{N}(A)}}_\ell$\\
	Hence : \\
	\begin{equation}\label{pl1}
	\begin{split}
	&\norm{(\boldsymbol{\hat{A}} -\boldsymbol{A})^{\mathcal{N}(A)}}_\ell + \norm{\boldsymbol{A}}_\ell  - \norm{\boldsymbol{\hat{A}}}_\ell \\
	&\overset{(*)}{=} \norm{\boldsymbol{A}}_\ell  - \norm{\boldsymbol{\hat{A}}^{\mathcal{N}_{\perp}(\boldsymbol{A})}}_\ell \overset{(**)}{\leq} \norm{(\boldsymbol{\hat{A}}-\boldsymbol{A})^{\mathcal{N}_\perp(\boldsymbol{A})}}_\ell
	\end{split}
	\end{equation}
	where $(*)$ follows Eq.(\ref{assuml1}) and $(**)$ follows
	the fact that $\left|\norm{\boldsymbol{A}}-\norm{\boldsymbol{B}}\right| \leq \norm{\boldsymbol{A} -\boldsymbol{B}} $.\\
	With Eq.(\ref{assuml1}) and Eq.(\ref{pl1}),  the correctness of lemma 1 is straightforward
\end{proof}
\begin{lem}
	Let $\chi^2(d)$ be a chi-square random variable with degree of freedom $d$, then the following inequality holds : \\
	$\mathbb{P}\left(\chi^2(d) \ge d+t \right) \le \dfrac{1}{\sqrt{2\pi Z_d(t)}}exp\left(-\dfrac{Z_d(t)}{2}\right) $,$\forall t > 0$ \\
	where $Z_d(t) = t- dlog(1+\dfrac{t}{d}) $, 
\end{lem}
\begin{proof}
	According to the Wallace inequality \cite{Wallace}, we have :
	\begin{equation}\label{res1l2}
	\begin{split}
	\mathbb{P}\left(\chi^2(d) \ge d+t \right) \le 
	\mathbb{P}\left(\mathcal{N}_{0,1} \ge \sqrt{Z_d(t)}\right)
	\end{split}
	\end{equation}
	where $\mathcal{N}_{0,1}$ is a random variable subject to $\mathcal{N}(0,1)$.\\
	Moreover, let $u = \sqrt{Z_d(t)}$ :\\
	\begin{equation}\label{res2l2}
	\begin{split}
	\mathbb{P}\left(\mathcal{N}_{0,1} \ge u\right)& =
	\int\limits_{u}^{+\infty}\dfrac{1}{\sqrt{2\pi}}exp\left(-\dfrac{x^2}{2}\right)dx  \\
	&\leq \int\limits_{u}^{+\infty}\dfrac{x}{u}\dfrac{1}{\sqrt{2\pi}}exp\left(-\dfrac{x^2}{2}\right)dx\\ & = \dfrac{1}{\sqrt{2\pi }u}exp\left(-\dfrac{u^2}{2}\right)
	\end{split}
	\end{equation}
	It is obvious that lemma 2 directly follows (\ref{res1l2}) and (\ref{res2l2}).
\end{proof}
\begin{lem} 
	Let $\alpha = 2\sigma\sqrt{dn_u+t}$, choose $\lambda_1$, $\lambda_2$, $\lambda_3 $ as : $\lambda_1 \ge {n_u}\alpha$,  $\lambda_2 \ge
	\tilde{n}\alpha$, $\lambda_3 \ge \alpha$
	, where $n_u = \sum\limits_{i=1}^{n_{a}}n_{u_i}$ and $\tilde{n} = \sqrt{\sum\limits_{i}n_{u_i}^2}$.
	let $\hat{\boldsymbol{W}} =(\hat{\boldsymbol{\theta}},\hat{\boldsymbol{P}},\hat{\boldsymbol{U}})$ be an optimal solution of $(P1)$, and $\boldsymbol{W} =({\boldsymbol{\theta}},{\boldsymbol{P}},{\boldsymbol{U}})$ be an arbitrary feasible solution such that $\boldsymbol{W} \neq \boldsymbol{
		\hat{W}}$, then 
	\begin{equation}\label{lem3}
	\begin{split}
	&\sum\limits_{i,j}\left\norm{\boldsymbol{x}^{(i,j)}\hat{\boldsymbol{w}}^{(i,j)} - \boldsymbol{f}^{*^{(i,j)}}\right}^2 \leq \sum\limits_{i,j}\left\norm{\boldsymbol{x}^{(i,j)}\boldsymbol{w}^{(i,j)} - \boldsymbol{f}^{*^{(i,j)}}\right}^2 \\ +
	&2\lambda_1 \left\norm{(\hat{\boldsymbol{\theta}}-\boldsymbol{\theta})^{\mathcal{N}_\perp(\boldsymbol{\theta})}\right}_1 +
	2\lambda_2\left\norm{(\hat{\boldsymbol{P}}-\boldsymbol{P})^{\mathcal{N}_\perp(\boldsymbol{P})}\right}_{1,2} \\
	+ &2\lambda_3\left\norm{(\hat{\boldsymbol{U}}^{\top}-\boldsymbol{U}^{\top})^{\mathcal{N}_\perp(\boldsymbol{U}^{\top})}\right}_{1,2} 
	\end{split}
	\end{equation}
	holds with probability at least:\\
	$1 - \dfrac{1}{\sqrt{2\pi Z_{dn_u}(t)}}exp\left(-\dfrac{Z_{dn_u}(t)}{2}\right) $, $\forall t > 0$
	where $n_u=\sum\limits_{i=1}^{n_a} n_{u_i}$
\end{lem}
\begin{proof}
	Since $\boldsymbol{\hat{W}}$ is an optimal solution of $(P1)$, for any feasible solution $\boldsymbol{W} \neq \boldsymbol{\hat{W}}$, we have :
	\begin{equation}
	\begin{split}
	&\sum\limits_{i,j}\left\norm{\boldsymbol{x}^{(i,j)}\hat{\boldsymbol{w}}^{(i,j)} - \boldsymbol{y}^{(i,j)}\right}^2 \leq \sum\limits_{i,j}\left\norm{\boldsymbol{x}^{(i,j)}\boldsymbol{w}^{(i,j)} - \boldsymbol{y}^{(i,j)}\right}^2 \\ &+
	\lambda_1 (\norm{{\boldsymbol{\theta}}}_{1}-\norm{\hat{\boldsymbol{\theta}}}_1) +
	\lambda_2(\norm{{\boldsymbol{P}}}_{1,2}-\norm{\hat{\boldsymbol{P}}}_{1,2})
	\\&+\lambda_3(\norm{{\boldsymbol{U}}^{\top}}_{1,2}-\norm{\hat{\boldsymbol{U}}^
		{\top}}_{1,2}) 
	\end{split}
	\end{equation}
	Let $\Delta{\boldsymbol{w}}_{ij} =\underbrace{(\boldsymbol{\hat{\theta}} - \boldsymbol{{\theta}})}_{\Delta{\boldsymbol{\theta}}} +\underbrace{(\boldsymbol{\hat{p}}^{(i)} - \boldsymbol{{p}}^{(i)})}_{\Delta\boldsymbol{p}^{(i)}}+ \underbrace{(\boldsymbol{\hat{u}}^{(i,j)} - 
		\boldsymbol{{u}}^{(i,j)})}_{\Delta\boldsymbol{u}^{(i,j)}} $, according to Eq.(4), we have :\\
	\begin{equation}\label{mainl3}
	\begin{split}
	&\sum\limits_{i,j}\left\norm{\boldsymbol{x}^{(i,j)}\hat{\boldsymbol{w}}^{(i,j)} - \boldsymbol{f}^{*^{(i,j)}}\right}^2 \leq \sum\limits_{i,j}\left\norm{\boldsymbol{x}^{(i,j)}\boldsymbol{w}^{(i,j)} - \boldsymbol{f}^{*^{(i,j)}}\right}^2 \\ 
	&+
	\lambda_1 (\norm{{\boldsymbol{\theta}}}_{1}-\norm{\hat{\boldsymbol{\theta}}}_1) +
	\lambda_2(\norm{{\boldsymbol{P}}}_{1,2}-\norm{\hat{\boldsymbol{P}}}_{1,2})
	\\&+\lambda_3(\norm{{\boldsymbol{U}}^{\top}}_{1,2}-\norm{\hat{\boldsymbol{U}}^
		{\top}}_{1,2}) + \underbrace{2\sum\limits_{i,j}\left<Z_{i,j},\Delta{\boldsymbol{w}}_{ij}\right>}_{\Delta L}
	\end{split}
	\end{equation}
	where $Z_{ij} = \boldsymbol{X}^{(i,j)^\top}\delta^{(i,j)} \in \mathbb{R}^d$ \\ Let $Z= [Z_{11},Z_{1,2},\cdots,Z_{n_a,n_{u_{n_a}}}]$,  we could bound $\Delta L$ further as:
	\begin{equation}\label{resdl}
	\begin{split}
	\Delta L &= 2\sum\limits_{i,j}\left<Z_{ij},\Delta\boldsymbol{\theta}\right>+ 2\sum\limits_{i,j}\left<Z_{ij},\Delta\boldsymbol{p}^{(i)}\right> \\ &+ 2 \sum\limits_{i,j} \left<Z_{ij},\Delta\boldsymbol{u}^{(i,j)}\right>\\
	&\overset{(a)}{\leq} 2\norm{Z}_F\Big(n_u\norm{\Delta\boldsymbol{\theta}}_1 +\tilde{n}\norm{\Delta\boldsymbol{P}}_{1,2} +  \norm{\Delta{U}^\top}_{1,2}\Big)
	\end{split}
	\end{equation}
	where (a) is the result of the fact that $ \norm{x}_2 \leq \norm{x}_1, \forall x \in \mathbb{R}^n$ and the cauchy-schwarz inequality.  \\
	For $Z_{ij,k}$ :the $k$th element of $Z_{ij}$, we have :
	\begin{equation*}
	\begin{split}
	Z_{ij,k} = \sum\limits_{l=1}^{n_{ij}}x^{(i,j)}_{lk}\delta^{(i,j)}_l
	\end{split}
	\end{equation*}
	Based on Eq.(1), it is easy to show that $\dfrac{Z^2_{ij,k}}{\sigma^2} \sim \chi^2(1)$. Hence, we get : 
	\begin{equation*}
	\begin{split}
	\dfrac{\norm{Z}_F^2}{\sigma^2} \sim \chi^2(dn_u)
	\end{split}
	\end{equation*}
	where $n_u = \sum\limits_{i=1}^{n_a}n_{u_i}$.\\
	It follows lemma 2  that 
	\begin{equation*}
	\begin{split}
	\mathbb{P}\left[2\norm{Z}_F \leq \alpha \right] \ge  1 - \dfrac{1}{\sqrt{2\pi Z_{dn_u}(t)}}exp\left(-\dfrac{Z_{dn_u}(t)}{2}\right) , \forall t > 0
	\end{split}
	\end{equation*}
	Suppose that $2\norm{Z}_F \leq \alpha$ holds, based on the requirement of this lemma and (\ref{resdl}), we attain :
	\begin{equation}\label{resl3L}
	\begin{split}
	\Delta L \leq \lambda_1\norm{\Delta\boldsymbol{\theta}}_1 + \lambda_2 \norm{\Delta\boldsymbol{P}}_{1,2}  + \lambda_3\norm{\Delta{U}^\top}_{1,2}
	\end{split}
	\end{equation}
	Now we could end the proof, since (\ref{lem3}) follows (\ref{mainl3}), (\ref{resl3L}), and lemma 1.
\end{proof}
Now we're ready to proof the main result based on assumption 1 and lemma 3.
\begin{Pft}
	Let $\boldsymbol{w}^{(i,j)} = \boldsymbol{w}^{*(i,j)}$, with lemma3, if (\ref{lem3}) holds, then we have :
	\begin{equation*}
	\begin{split}
	&\norm{\boldsymbol{X}\boldsymbol{\overline{W}- \boldsymbol{F}}}^2 \leq 2\lambda_1 \norm{(\boldsymbol{\hat{\boldsymbol{\theta}}} - \boldsymbol{\boldsymbol{\theta^*}})^{\mathbb{N}_\perp(\boldsymbol{\theta}^*)}}_{1} \\
	&+ 2\lambda_2\norm{(\boldsymbol{\hat{\boldsymbol{P}}} - \boldsymbol{\boldsymbol{P}^*})^{\mathbb{N}_\perp(\boldsymbol{P}^*)}}_{1,2}\\ & +2\lambda_3
	\norm{(\boldsymbol{\hat{\boldsymbol{U}}^{\top}} - \boldsymbol{\boldsymbol{U}^{*^{\top}}})^{\mathbb{N}_\perp(\boldsymbol{U}^{*^{\top}})}}_{1,2}
	\end{split}
	\end{equation*}
	Take $\Gamma_{\theta^*} = \boldsymbol{\hat{\theta}} - \boldsymbol{\theta}^*$, $\Gamma_{P^*} = \boldsymbol{\hat{P}} - \boldsymbol{P}^*$, $\Gamma_{U^*} = \boldsymbol{\hat{U}} - \boldsymbol{U}^*$, $\Gamma_{\bar{W}} = \hat{\overline{W}} - \overline{W}^*$
	according to assumption 1, we attain :\\
	\begin{equation}\label{rest}
	\begin{split}
	&\norm{(\boldsymbol{\hat{\boldsymbol{\theta}}} - \boldsymbol{\boldsymbol{\theta^*}})^{\mathbb{N}_\perp(\boldsymbol{\theta}^*)}}_{1}  \leq \sqrt{n_\theta} \norm{(\boldsymbol{\hat{\boldsymbol{\theta}}} - \boldsymbol{\boldsymbol{\theta^*}})^{\mathbb{N}_\perp(\boldsymbol{\theta}^*)}}_{2} \\
	&\leq
	\dfrac{\sqrt{n_\theta}}{\kappa_\theta\sqrt{n_{min}n_u}}\norm{\boldsymbol{X}\boldsymbol{\overline{W}- \boldsymbol{F}}}
	\end{split}
	\end{equation}
	Similarly, we could reach :
	\begin{equation}\label{resp}
	\begin{split}
	\norm{(\boldsymbol{\hat{\boldsymbol{P}}} - \boldsymbol{\boldsymbol{P}^*})^{\mathbb{N}_\perp(\boldsymbol{P}^*)}}_{1,2} \leq
	\dfrac{\sqrt{n_p}}{\kappa_p\sqrt{n_{min}n_u}}\norm{\boldsymbol{X}\boldsymbol{\overline{W}- \boldsymbol{F}}}
	\end{split}
	\end{equation}
	\begin{equation}\label{resu}
	\begin{split}
	&\norm{(\boldsymbol{\hat{\boldsymbol{U}}^{{\top}}} - \boldsymbol{\boldsymbol{U}^{*^{\top}}})^{\mathbb{N}_\perp(\boldsymbol{U}^{*^{\top}})}}_{1,2} \\ & \leq
	\dfrac{\sqrt{n_{u,a}}}{\kappa_{u,a}\sqrt{n_{min}n_u}}\norm{\boldsymbol{X}\boldsymbol{\overline{W}- \boldsymbol{F}}}
	\end{split}
	\end{equation}
	Combining lemma3, (\ref{rest})-(\ref{resu}), we could reach that (12) holds.\\
	Following assumption 1, we get :
	\begin{equation}\label{btheta}
	\begin{split}
	\norm{(\boldsymbol{\hat{\boldsymbol{\theta}}} - \boldsymbol{\boldsymbol{\theta^*}})}_{1}  \leq (\beta_\theta+1)  \norm{(\boldsymbol{\hat{\boldsymbol{\theta}}} - \boldsymbol{\boldsymbol{\theta^*}})^{\mathbb{N}_\perp(\boldsymbol{\theta}^*)}}_{1} 
	\end{split}
	\end{equation}
	\begin{equation}\label{bp}
	\begin{split}
	\norm{(\boldsymbol{\hat{\boldsymbol{P}}} - \boldsymbol{\boldsymbol{P}^*})}_{1,2} \leq (\beta_p+1)
	\norm{(\boldsymbol{\hat{\boldsymbol{P}}} - \boldsymbol{\boldsymbol{P}^*})^{\mathbb{N}_\perp(\boldsymbol{P}^*)}}_{1,2} 
	\end{split}
	\end{equation}
	\begin{equation}\label{bu}
	\begin{split}
	&\norm{(\boldsymbol{\hat{\boldsymbol{U}}^{{\top}}} - \boldsymbol{\boldsymbol{U}^{*^{\top}}})}_{1,2}\\ &\leq (\beta_{u,a}+1) \norm{(\boldsymbol{\hat{\boldsymbol{U}}^{{\top}}} - \boldsymbol{\boldsymbol{U}^{*^{\top}}})^{\mathbb{N}_\perp(\boldsymbol{U}^{*^{\top}})}}_{1,2}
	\end{split}
	\end{equation}
	Then (13)-(15)  directly follows lemma 3 and (\ref{btheta})-(\ref{bu}).\qed
\end{Pft}
\end{document}